\documentclass[sigconf, anonymous=False,review=False]{acmart}

\AtBeginDocument{%
  \providecommand\BibTeX{{%
    \normalfont B\kern-0.5em{\scshape i\kern-0.25em b}\kern-0.8em\TeX}}}

\copyrightyear{}
\acmYear{}
\acmConference[]{}{}{}
\acmBooktitle{}
\acmPrice{}
\acmDOI{}
\acmISBN{}

\usepackage{float}
\usepackage{physics}
\usepackage{tabularx}
    \newcolumntype{Y}{>{\centering\arraybackslash}X}
\usepackage{booktabs}
\usepackage{pdflscape}
\usepackage{csvsimple}
\usepackage{hyperref}[hyperfootnotes=true]
\usepackage{url}
\usepackage{nameref}
\usepackage{tablefootnote}
\usepackage{color, colortbl}
\usepackage{mathtools}
\usepackage{xparse}
\usepackage[vlined,linesnumbered]{algorithm2e}
\RestyleAlgo{ruled}
\usepackage{makecell,booktabs}
\usepackage{graphicx}
\newcommand{\ignore}[1]{}

\usepackage
[
    noabbrev,   % The words in front of the labels are not abbreviated.
    nameinlink, % Extends the link of a reference to the word in front of it.
]
{cleveref} % This package must be included after ›hyperref‹. It creates clever references that know what they refer to.

\DeclareDocumentCommand{\set}{m g o}{
    \ensuremath{
        \IfNoValueTF{#3}{\left}{#3}\{#1
            \IfNoValueTF{#2}{}{
                \ \IfNoValueTF{#3}{\left}{#3}\vert\ \vphantom{#1}#2\IfNoValueTF{#3}{\right.}{}
            } \IfNoValueTF{#3}{\right}{#3}\}
    }\xspace
}

\definecolor{lightgreen}{rgb}{0.75,0.92,0.61}

\newcommand{\floor}[1]{\left\lfloor #1 \right\rfloor}

\newcommand{\labelname}[1]{
  \def\@currentlabelname{#1}}

\newcommand{\R}{\mathbb{R}}
\newcommand{\Z}{\mathbb{Z}}

\providecommand{\ignore}[1]{} 

\newcommand{\ooea}{(1+1) EA\xspace}
\newcommand{\rlswself}{$RLS_{\alpha,\beta}$\xspace}

\newcommand{\realnum}{\mathbb{R}}
\newcommand{\natnum}{\mathbb{N}}

\newcommand{\intstring}{\mathbb{Z}^n}

\newtheorem{thm}{Theorem}
\newtheorem{lem}[thm]{Lemma}

\newtheorem{definition}[thm]{Definition}
\newtheorem{cor}[thm]{Corollary}

	{
	\begin{center}
	\begin{algorithm2e}
	}%
	{
	\end{algorithm2e}
	\end{center}
	}
	
\allowdisplaybreaks
\allowdisplaybreaks

\begin{document}

\title{Run Time Bounds for Integer-Valued OneMax Functions}

\author{Jonathan Gadea Harder}
\email{ jonathan.gadeaharder@hpi.de}
\affiliation{%
  \institution{Hasso Plattner Institute\\ University of Potsdam}
  \city{Potsdam}
  \country{Germany}
}

\author{Timo Kötzing}
\email{timo.koetzing@hpi.de}
\affiliation{%
  \institution{Hasso Plattner Institute\\ University of Potsdam}
  \city{Potsdam}
  \country{Germany}
}

\author{Xiaoyue Li}
\email{xiaoyue.li@hpi.de}
\affiliation{%
  \institution{Hasso Plattner Institute\\ University of Potsdam}
  \city{Potsdam}
  \country{Germany}
}

\author{Aishwarya Radhakrishnan}
\email{aishwarya.radhakrishnan@hpi.de}
\affiliation{%
  \institution{Hasso Plattner Institute\\ University of Potsdam}
  \city{Potsdam}
  \country{Germany}
}

\renewcommand{\shortauthors}{ Kötzing, et al.}
\begin{abstract}
While most theoretical run time analyses of discrete randomized search heuristics focused on finite search spaces, we consider the search space $\Z^n$. This is a further generalization of the search space of multi-valued decision variables $\{0,\ldots,r-1\}^n$.

We consider as fitness functions the distance to the (unique) non-zero optimum $a$ (based on the $L_1$-metric) and the \ooea which mutates by applying a step-operator on each component that is determined to be varied. For changing by $\pm 1$, we show that the expected optimization time is $\Theta(n \cdot (|a|_{\infty} + \log(|a|_H)))$. In particular, the time is linear in the maximum value of the optimum $a$. Employing a different step operator which chooses a step size from a distribution so heavy-tailed that the expectation is infinite, we get an optimization time of $O(n \cdot \log^2 (|a|_1) \cdot \left(\log (\log (|a|_1))\right)^{1 + \epsilon})$. 

Furthermore, we show that RLS with step size adaptation achieves an optimization time of $\Theta(n \cdot \log(|a|_1))$.

We conclude with an empirical analysis, comparing the above algorithms also with a variant of CMA-ES for discrete search spaces.
\end{abstract}

\ignore{

}



\maketitle

%\Tcomment{TODO for the final version: add acknowledgment of DFG as in GECCO papers}
%\Scomment{Added. But not appears after compiling}

\section{Introduction}

Optimization problems are formalized as finding the optimal element $x$ from a fixed search space $\mathcal{X}$ given some quality measure $f: \mathcal{X} \rightarrow \realnum$. In the theory of evolutionary search heuristics, the most commonly studied discrete search space is $\mathcal{X} = \{0,1\}^n$, the set of bit strings of fixed length. Other search spaces have been considered, such as permutations \cite{frank21, doerr22}, or multi-valued decision variables \cite{gunter94,timo2017} ($\mathcal{X} = \{0,\ldots,r-1\}^n$). Note that all these search spaces are finite.

In this work we are interested in the infinite (but still discrete) search space $\Z^n$. This models a set of $n$ decision variables with an infinite (totally and discretely ordered) domain. While many search problems can be usefully addressed by translating them into optimization problems using $\mathcal{X} = \{0,1\}^n$ or another of the before-mentioned search spaces, this is impossible in principle for an infinite search space. Furthermore, understanding them in their more natural formulation can lead to more efficient optimization algorithms.

In order to analyze heuristic search in this domain, we generalize fitness functions as well as heuristic search algorithms accordingly. Note that a generalization of the $\{0,1\}^n$ search space to $\{0,\ldots,r-1\}^n$ was done in \cite{timo2017}, and we follow a similar path in the generalization, but now with the added difficulty of an infinite search space.

As a first analysis, we consider the simple setting where, for a given $a \in \Z^n$, we have the fitness function
$$
f_a : \Z^n \rightarrow \Z_{\geq 0}, x \mapsto \sum_{i=1}^n |x_i-a_i|.
$$
Minimizing this function generalizes the well-known OneMax function (class), defined on the search space $\{0,1\}^n$, to the more general space of $\Z^n$.

As algorithms, we consider the \ooea and RLS (Random Local Search), both suitably adjusted to deal with the search space $\Z^n$ as follows (see Section~\ref{sec:prelims} for a detailed description of both algorithms). First, for finite search spaces, it is common to start the search with a uniformly random search point. For the infinite search space $\Z^n$ we make the decision to start deterministically with the all-$0$ string. Since for our fitness functions only the position-wise differences of the starting point to the optimum matter, this choice does not restrict the meaningfulness of our results.

Second, as a variation operator, we consider to change either exactly one position (RLS) or each position independently with probability $1/n$ (\ooea), just as in the common definitions of these algorithms. However, while changing a \emph{bit} from $\{0,1\}$ leaves only one possible choice for the new value, changing a variable from $\Z$ leaves infinitely many choices for the new value. We consider two possible \emph{step operators}, defining how to change a value from $\Z$. The first step operator is the \emph{$\pm 1$ operator}; it either increases or decreases the value by $1$, decided uniformly at random. The second step operator is the \emph{heavy-tailed operator}; it makes a uniform decision to either increase or decrease, but instead of deterministically changing by 1, it changes by a random number. In particular, we consider a distribution of the numbers that is unbounded and is so heavy-tailed that it does not have finite expectation. 

Note that \cite{timo2017} considered a \emph{Harmonic operator} as a step operator for variables on $\{0,\ldots,r-1\}$, which gives a step of size $i$ a weight proportional to $1/i$. This cannot be directly extended to an operator on $\Z$, since the sequence $(1/i)_i$ is not summable. Instead, we take inspiration from \cite{unknown_solution_length}, where very slowly decreasing yet summable sequences were considered, to define our heavy-tailed operator. In effect, a step size of $i$ has a probability proportional to $1/(i \log(i)^{i+\varepsilon})$. The goal of this very heavy tail of the distribution is to have as high as possible a chance to gain a constant fraction of the distance to the optimum, independent of the distance to the optimum. See Section~\ref{sec:prelims} for details on the distribution.

Recently, heavy tailed distributions where used to speed up optimization in various settings. In \cite{fastMutator2017}, the authors proposed to apply a heavy-tailed mutation operator for the first time. In particular, the number of variables to change was chosen from a heavy-tailed distribution (which for us still follows the traditional binomial distribution). In contrast to our work, the heavy-tailed distributions in \cite{fastMutator2017} have finite expectation. This more explorative mutation operator was then shown to optimize so-called jump-functions faster. Since this work, further analyses have shown the use of such heavy-tailed distributions, for example for crossover and the $(1+(\lambda,\lambda))$-GA on OneMax \cite{fastMutatorCrossover} and jump-functions \cite{BenjaminHeavy-TailedGA}.

In our work we consider the expected time of the given algorithms to find the optimum of a fitness function $f_a$. This time naturally depends on $a$ (as well as $n$), specifically on its total weight $|a|_1 = \sum_{i=1}^n |a_i|$, its maximal weight $|a|_{\infty} = \max\set{|a_i|}{i \leq n}$ or its Hamming distance to the all-$0$ string $|a|_H = |\set{i}{a_i \neq 0, i \leq n}|$. Through out our analyses we consider $a$ to be non-zero.

In Section~\ref{sec:unit_mutation_operator} we formally analyze the \ooea with the $\pm 1$ operator. Theorems~\ref{thm:better_upperbound_pm_ operator} and~\ref{thm:better_lowerbound_pm_operator} show that the expected optimization time is $\Theta(n \cdot (|a|_{\infty} + \log(|a|_H)))$ on any given fitness function $f_a$. While the linear dependence on the dimension $n$ constitutes a good performance, the linear dependence on $|a|_{\infty}$ is rather slow. For comparison, note that there are $\Theta(|a|_{\infty}^n)$ many target bit strings of roughly that size. Since the \ooea gains about one bit of information with a comparison of two fitness values, a direct information theoretic lower bound for finding the optimum is at $\Omega(n \cdot \log (|a|_{\infty}))$.

In Section~\ref{sec:heavy_tailed_operator} we turn to the \ooea with the heavy-tailed operator. In Theorem~\ref{thm:heavy_tailed_operator} we show that the expected optimization time is $O(n \cdot \log^2 (|a|_1) \cdot \left(\log (\log (|a|_1))\right)^{1 + \epsilon})$ for a given $f_a$. This is already much closer to the information-theoretic lower bound mentioned above. 

Finally, we consider a version of RLS which adapts the step size during the search. This strategy was proven to be very efficient in \cite{timo2017} for the search space $\{0,\ldots,r-1\}$ and we show that also here the algorithm achieves an expected optimization time of $\Theta(n \cdot \log (|a|_1))$. Note that, for all three operators, we derive central parts of our proofs by carefully adjusting analogous proofs from~\cite{timo2017}.

With our analyses of the search space $\Z^n$, we also aim at bridging the gap towards analyses of heuristic search on $\R^n$ (continuous optimization): if one is interested in approximating the optimum up to a distance of $\varepsilon$, one can discretize the search space accordingly, arriving at the search space $\Z^n$. In optimization problems on $\R^n$, one is frequently interested only in approximating the optimum, since finding it is typically impossible in principle. In Corollary~\ref{cor:heavy_tailed_operator} we show a result about the \ooea with heavy-tailed operator approximating the optimum, showing that finding an approximation ratio of $\alpha$ scales with respect to $\alpha$ as $O(\log(1/\alpha))$.

For $\R^n$, the covariance matrix adaptation evolution strategy (CMA-ES) is a widely used optimization algorithm \cite{CMA-ES}. To extend the use of CMA-ES to address also discrete variables, authors in \cite{CMAESwMargin} propose a variant of CMA-ES (called CMA-ES with margin, CMA-ESwM). We are interested in the performance of this algorithm in optimizing an integer-valued problem in comparison to other discrete optimization algorithms. In Section~\ref{sec:experiments} we experimentally compare our algorithms with CMA-ESwM, finding that the CMA-ESwM fails to optimize with certain probability even with a large time budget, while the \ooea with the heavy-tailed operator and RLS with the self-adjusting operator can handle the instances efficiently.

The remainder of this paper is structured as follows. In Section~\ref{sec:prelims} we introduce algorithms and notation. In Sections~\ref{sec:unit_mutation_operator} and~\ref{sec:heavy_tailed_operator} we give our theoretical analyses of the \ooea. Section~\ref{sec:self_adjusting_rls} addresses the self-adjusting RLS. In Section~\ref{sec:experiments} we present our experiments.

\section{Preliminaries} \label{sec:prelims}

In this section we define the \ooea and random local search algorithms, along with the different step operators. At the end of this section we also state different drift theorems we use for our analysis.

For any $a \in \intstring$, we let 
\begin{align*}
|a|_1 &= \sum_{i = 1}^{n} |a_i|;\\
|a|_{\infty} &= \max\set{|a_i|}{i \leq n};\\
|a|_H &= |\set{i}{a_i \neq 0, i \leq n}|.
\end{align*}
Also for $a \in \intstring \setminus \{0^n\}$, we define 
$$
f_a: \intstring \rightarrow \realnum, x  \mapsto |a - x|_1.
$$
Our class of target fitness functions is $\mathscr{F} = \{f_a \mid a \in \intstring \setminus \{0^n\} \}$ to be minimized by evaluating the fitness function at any point as chosen by the algorithm.

We consider different \emph{step operators} $\mathrm{step}: \Z\mapsto \Z$. These step operators decide on the update of a mutation in a given component. We consider the following step operators.

\begin{definition}\label{defn:pm_operator}
[$\pm 1$-operator] The \emph{$\pm 1$-operator} takes an integer $x$ as input and makes the following changes: With probability $1/2$ return $x+1$ and otherwise return $x-1$.
\end{definition}

For the next operator we need a definition. Given $\epsilon > 0$, let $ c_{\epsilon} = \sum_{i = 2}^{\infty} \frac{1}{i \cdot (\log i)^{1+\epsilon}}$ (note that this sum is finite \cite{unknown_solution_length}).

\begin{definition}\label{defn:heavy_tailed_operator}
[Heavy-tailed-operator] For a given $\varepsilon > 0$, the \emph{heavy-tailed operator} takes an integer $x$ as input and makes the following changes: First sample a step size of $2^{I - 2}$ using a random variable $I$ which can take a value of any natural number $i \geq 2$ with $P(I=i) = \frac{1}{c_{\epsilon} \cdot i \cdot (\log i)^{1+\epsilon}}$. With probability $\frac{1}{2}$ then return $x+2^{I - 2}$, otherwise return $x-2^{I - 2}$.
\end{definition}

We consider the algorithms RLS and the \ooea as given by \Cref{1+1 ea} and \Cref{rls}. Both start from the initial search point being the all-$0$ string. They then proceed in rounds, each of which consists of a \emph{mutation} and a \emph{selection} step. Throughout the whole optimization process, the algorithms maintain a single individual, which is always the most recently sampled best-so-far solution. The two algorithms differ only in the \emph{mutation operator}. While $RLS$ makes a step in exactly one position (chosen uniformly at random), the \ooea makes, in each position, a step with probability $1/n$. We specify the termination criterion as the point in time when the search point has a fitness of $0$.

\SetKw{KwFrom}{from}

\begin{algorithm}
\textbf{Initialization:} $x\gets 0^n$\;
 \textbf{Optimization:} \While{$f(x) \neq 0$}{
 \For{i \KwFrom 1 \KwTo n}{
 With probability $\frac{1}{n}$ set $y_i=\text{step}(x_i)$ and set\\ $y_i=x_i$ otherwise;
 }
\If{$f(y)\leq f(x)$}{
   $x \gets y$;
   }

 }
 \caption{The \ooea minimizing a function $f:\mathbb{Z}^n \mapsto\mathbb{R}$ with a given step operator $\mathrm{step}:\Z \rightarrow \Z$}
 \label{1+1 ea}
\end{algorithm}

$RLS_{\alpha,\beta}$ maintains a search point $x\in \mathbb{Z}^n$ as well as a real-valued \emph{velocity vector} $v \in [1,\infty]^n$; we use real values for the velocity to circumvent rounding problems. The initial search point is the all-$0$s string and the initial velocity is the all-$1$s string. In one iteration of the algorithm a position $i\in [n]$ is chosen uniformly at random. The entry $x_i$ is replaced by $x_i-\floor{v_i}$ with probability $1/2$ and by $x_i+\floor{v_i}$ otherwise. The entries in positions $j\neq i$ are not subject to mutation. The resulting string $y$ replaces $x$ if its fitness is at least as good as the one of $x$, i.e.\ if $f(y)\leq f(x)$ holds. If the offspring $y$ is strictly better than its parent $x$, i.e.\ if $f(y) < f(x)$, we increase the velocity $v_i$ in the $i$-th component by multiplying it with the constant $\alpha$; we decrease $v_i$ to $\beta v_i$ otherwise. The algorithm proceeds this way until we decide to stop it. To further lighten the notation, we say that the algorithm ``\emph{moves in the right direction}'' or ``\emph{towards the target value}'' if the distance to the target is actually decreased by $\floor{v_i}$. Analogously, we speak otherwise of a step ``\emph{away from the target}'' or ``\emph{in the wrong direction}''.

\begin{algorithm}
\textbf{Initialization:} $x\gets 0^n$, $v \gets 1^n$\;
 \textbf{Optimization:} \While{$f(x) \neq 0$}{
 \For{i \KwFrom 1 \KwTo n}{
 $y\gets x$;\\
 Choose $i\in [n]$ uniformly at random;\\
 With probability $\frac{1}{2}$ set $y_i=x_i-\floor{v_i}$ and set $y_i=x_i+\floor{v_i}$ otherwise;
 }
 \uIf{$f(y)<f(x)$}{
 $v_i \gets \alpha v_i$;
 }\Else{
 $v_i\gets \max\{1,\beta v_i\}$;
 }
\If{$f(y)\leq f(x)$}{
   $x \gets y$;
   }
 }
 \caption{$RLS_{\alpha,\beta}$ with self-adjusting step sizes minimizing a function $f:\mathbb{Z}^n\mapsto \mathbb{R}$}
 \label{rls}
\end{algorithm}

A central tool in many of our proofs is drift analysis, which comprises a number of tools to derive bounds on hitting times from bounds on the expected progress a process makes towards the target. Drift analysis is currently the most powerful tool in run time analysis for evolutionary computation. We briefly collect here the tools we use.

\emph{Additive drift} is the situation that the progress is bounded by any (constant) value. This quite common situation in run time analysis was the first framed into a drift theorem, namely the following one, in \cite{driftthm}.

% \Acomment{Timo, Additive drift theorem is not the actual version in the paper mentioned as reference, should we change it to the one in the paper\cite{driftthm}?} \Tcomment{if its just a modern formulation, just ignore; if you took the formulation from somewhere else, give credit to that other source as well. The drift theorem is incorrect as it is, since the process may be negative}

\begin{thm}[Additive Drift Theorem \cite{driftthm}] \label{thm:additivedrift}
Let $S \subseteq \realnum$ be a finite set of positive numbers and let $({X^t})_{t\in \natnum}$ be a sequence of random variables over $S \cup \{0\}$. Let $T$ be the random variable that denotes the first point in time $t \in \natnum$ for which $X^t \leq 0$.
Suppose that there exists a constant $\delta_1 > 0$ such that

$$ E[X^{t} - X^{t+1} \mid T > t] \geq \delta_1 $$
holds. Then

$$E[T \mid X^0] \leq \frac{X^0}{\delta_1}.$$

If there exists a constant $\delta_2 > 0$ such that

$$ E[X^{t} - X^{t+1} \mid T > t] \leq \delta_2 $$
holds. Then

$$E[T \mid X^0] \geq \frac{X^0}{\delta_2}.$$

\end{thm}

\emph{Multiplicative drift} (first) addresses the situation where progress is proportional to the distance from the target. For this situation quite common in run time analysis we can use the following drift theorem~\cite{multiplicative-drift}.

\begin{thm}[Multiplicative Drift Theorem \cite{firsthit}] \label{thm:multiplicativedrift}
Let $(X_t)_{t \in \natnum}$ be random variables over $\realnum$, $x_{\min} >0$, and let $T = \min\{t \mid X_t < x_{\min}\}$. Furthermore, suppose that\\

(a) $X_0 \geq x_{\min}$ and, for all $t \leq T$, it holds that $X_t \geq 0$, and that

(b) there is some value $\delta >0$ such that, for all $t < T$, it holds that $X_t - E[X_{t+1} \mid X_0, \ldots, X_t] \geq \delta X_t$.\\

Then

$$ E[T \mid X_0] \leq \frac{1+\ln\left(\frac{X_0}{x_{\min}}\right)}{\delta}.$$

% Let $S \subseteq \realnum$ be a finite set of positive numbers with minimum $s_{min}$. Let $({X^t})_{t\in \natnum}$ be a sequence of random variables over $S \cup \{0\}$. Let $T$ be the random variable that denotes the first point in time $t \in \natnum$ for which $X^t < 1$.
% Suppose that there exists a real number $\delta > 0$ such that

% $$ E[X^{t} - X^{t+1} \mid X^{t} = s] \geq \delta s $$
% holds for all $s \in S$ with $Pr[X^t = s] > 0$. Then for all $s_0 \in S$ with $Pr[X^{0} = s_0] > 0$, we have

% $$E[T \mid X^0 = s_0] \leq \frac{1 + \ln(s_0/s_{min})}{\delta}.$$
\end{thm}
% \Tcomment{Twice you write $X(?)$ instead of using a superscript}
% \Tcomment{$ln$ should be $\ln$.}

% Any logarithm for which a base is not explicitly mentioned is the logarithm with base 2.\Tcomment{But $\ln$ is base $e$, no?}

\section{Unit Mutation Strength} \label{sec:unit_mutation_operator}
In this section, we regard the mutation operator that applies only $\pm 1$ changes to each component. We give a tight bound on the expected run time of the \ooea with the $\pm$ operator. We start by stating the following lemma which bounds the expected cost of a given random variable over $[n]$ for $n \in \natnum$, which we later use in our analysis.

\begin{lem}\cite[Lemma 13]{timo2017} \label{cost_bound}
Let $n \in \natnum$ be fixed, let $q$ be a cost function on elements of $[n]$ and let $c$ be a cost function on subsets of $[n]$. Furthermore, let a random variable $S$ ranging over subsets of $[n]$ be given. Then we have

\begin{align}
    \forall T\subseteq [n]: c(T)\leq \sum_{i\in T} q(i)\implies E[c(S)]\leq \sum_{i=1}^n q(i)P(i\in S)\end{align}
    and \begin{align}
    \forall T\subseteq [n]:c(T)\geq \sum_{i\in T} q(i)\implies E[c(S)]\geq \sum_{i=1}^n q(i)P(i\in S)
\end{align}
\end{lem}

The following theorem gives an upper bound on the expected optimization time of the $\ooea$ optimizing any $f_a \in \mathscr{F}$ with the $\pm 1$ operator.

\begin{thm} \label{thm:better_upperbound_pm_ operator}
     Let $f_a \in \mathscr{F}$. Then the expected optimization time of the \ooea with $\pm 1$ operator starting with all $0$ integer string on $f_a$ is $O(n\cdot(|a|_{\infty} +  \log(|a|_H)))$.
\end{thm}

\begin{proof}
Our proof idea is similar to the proof idea in \cite[Theorem 12]{timo2017}. We  will use the multiplicative drift analysis (see \Cref{thm:multiplicativedrift}) to prove this theorem. For our analysis, we make use of two edge cases that both have a fitness of $n$; which would be only one entry being incorrect but $n$ away from its target and the case that every entry in the target is $1$. The former is hard to optimize since it takes long for the algorithm to successively change the incorrect position, while the latter can be solved very fast since there are multiple ways of progressing in each step. We exploit this by giving each position a weight exponential in the amount that is incorrect, and then sum over those weights.
With any search point $x \in\Omega$ we associate a vector $d \in \mathrm{R}^n$ such that, for all $i \leq n, d_i = |a_i - x_i|$.
Given some $\omega>1$ we consider the potential 
\begin{align}
    g(x)\coloneqq \sum_{i=1}^{n} (\omega^{d_i}-1)
\end{align}

Let $x_t$ be the integer string at iteration $t$ when the \ooea with $\pm 1$ operator is optimizing $f_a$. Let $X_t = g(x_t)$ and let $T = \min\{t \geq 0 \mid X_t = 0\}$. Further let $E_1$ be the event that $X_{t+1}$ is obtained by mutating exactly one position and let $E_2$ be the event that $X_{t+1}$ is obtained by mutating at least two positions. Then $X_0 \leq |a|_H\cdot  \omega^{|a|_{\infty}}$, since we start with all $0$ integer string.
We denote $O$ as the set of already optimized positions and $A(S)$ as the set of accepted offspring by only manipulating $S\subseteq [n]$ positions.
Now we bound the expected drift. Since $t < T$, at least one of the position is at least $1$ distant far from the optimum and the probability to mutate this position in the right direction is $\frac{1}{2n}$ and the probability to not mutate any other positions is $\left(1 - \frac{1}{n}\right)^{n - 1}$. Therefore,
\begin{align*}
E[X_{t} - X_{t+1}& \mid t < T, E_1]\cdot P(E_1)\\ &\geq \sum_{i\in [n]\setminus O} \frac{1}{2n}\left(1 - \frac{1}{n}\right)^{n - 1} (\omega^{d_i}-1)-(\omega^{d_i-1}-1) \\
&\geq \frac{1}{2ne}\sum_{i\in [n]\setminus O} (\omega^{d_i}-1)-(\omega^{d_i-1}-1)\\
&= \frac{1}{2ne}\sum_{i\in [n]\setminus O}^{n} \left(1-\frac{1}{\omega}\right) \cdot \omega^{d_i}\\
&\geq \frac{1}{2ne}\sum_{i=1}^{n} \left(1-\frac{1}{\omega}\right) \cdot \left(\omega^{d_i}-1\right)\\
&=  \frac{\omega-1}{2\omega ne}\sum_{i=1}^{n} (\omega^{d_i}-1)
\end{align*}
Let $U$ be the set of all tuples $(y,i),y\in A(S)$ where the $i$-th positions worsens. As we only consider accepted mutations, we have that, for all $y \in A(S), \sum_{i\in S} d(y_i , z_i) - d_i \leq 0$.This implies that there are at least as many improving-pairs as there are worsening-pairs in $A(S) \cross S$.  For every $(y,i)$ with $d_i=0$ where bit position $i$ changes in the wrong direction, there is a $(y',i)\in U$ with bit position $i$ changing in the right direction and the remaining positions behaving the same. The change in potential for both pairs added is $$\omega^{d_i+1}-\omega^{d_i}+\omega^{d_i-1}-\omega^{d_i}= \omega^{d_i}\frac{(\omega-1)^2}{\omega}$$
Since there are in total at least as many improving-pairs as worsening-pairs, we can further map injectively each $y\in A(S)$ with a correct position ($d_i=0$) that changes in the wrong direction to another $y' \in A(S)$ with some improving position. The change in potential for both positions added is $$\omega-1+\omega^{d_{i'}-1}-\omega^{d_{i
'}}= (1-\omega^{d_i'-1})\cdot (\omega-1)< \omega^{d_i}\frac{(\omega-1)^2}{\omega}$$

 Let $Y$ be the random variable describing the search point after one cycle of mutation and selection. The random variable Y is completely determined by choosing a set $S \subseteq [n]$ of bit positions to change in $x$ and then, for each such position $i \in S$, choosing whether to change towards or away from the target. For each possible $S \subseteq [n]$, let $Y(S)$ be the random variable $Y$ conditional on making changes exactly at the bit positions of S. Note that
since we increase/decrease each index by $1$ with the same probability, $Y(S)$ is the uniform distribution on $A(S)$. Further let \begin{align*}
    c(S) &\coloneq E[g(Y(S))-g(x)] \\
    &=\frac{1}{|A(S)|}\sum_{y\in A(S)} g(y)-g(x)\\
    &= \frac{1}{|A(S)|}\sum_{y\in A(S)}\sum_{i=1}^n \left(\omega^{|y_i-a_i|}-\omega^{d_i}\right)\\
    &= \frac{1}{|A(S)|}\sum_{(y,i)\in U} \left(\omega^{|y_i-a_i|}-\omega^{d_i}+\omega^{|y_{i'}-a_{i'}|}-\omega^{d_{i'}}\right)\\
    &\leq \frac{1}{2}\sum_{i\in S} \omega^{d_i}\frac{(\omega-1)^2}{\omega}.
    \end{align*}

    Using Lemma \ref{cost_bound}, we get that
    $$E[g(Y)-g(x)|E_2]\leq \sum_{i=1}^n \frac{1}{n}\omega^{d_i}\frac{(\omega-1)^2}{2\omega}=\frac{(\omega-1)^2}{2\omega n}\sum_{i=1}^n \omega^{d_i}$$
We can use any $\omega>1$ such that $\omega-1-e(\omega-1)^2 >0$ and set $c=(\omega-1-e(\omega-1)^2)/e$. One can verify that for $\omega=1.2$ we obtain $0.0912687>0$, so such $\omega$ exist. In total we get
\begin{align*}
    E[g(x)-g(Y)]\geq \frac{\omega-1}{2\omega ne}\sum_{i=1}^n \omega^{d_i}-\frac{(\omega-1)^2}{2\omega n}\sum_{i=1}^n\omega^{d_i}\\
    =\frac{\omega-1-e(\omega-1)^2}{2\omega n e}\sum_{i=1}^n \omega^{d_i}=\frac{c}{2\omega n}\sum_{i=1}^n\omega^{d_i}\geq \frac{c}{2\omega n} g(x).
\end{align*}
Therefore by the multiplicative drift theorem (\Cref{thm:multiplicativedrift}), we have
\[E[T]  = O((|a|_{\infty}+ \log(|a|_H))\cdot n)= O(|a|_{\infty} \cdot n + n\cdot \log(|a|_H)).\]
\end{proof}

We show in the following theorem that this upper bound on the \ooea is sharp by proving the same asymptotic lower bound.
\begin{thm} \label{thm:better_lowerbound_pm_operator}
Let $f_a \in \mathscr{F}$. Then the expected optimization time of the \ooea with $\pm 1$ operator starting with all $0$ integer string on $f_a$ is $\Omega(|a|_{\infty} \cdot n + n\cdot \log(|a|_H))$.
\end{thm}

\begin{proof}
    The lower bound $|a|_{\infty} \cdot n$ follows from looking at the position with a distance of $|a|_{\infty}$ to the target. The drift in the right direction is at most $\frac{1}{n}$, which is the probability of mutating the position. Therefore by the additive drift theorem (\Cref{thm:additivedrift}), we have a run time of  $\Omega(n\cdot |a|_{\infty})$. For the second part we prove that $P(T \geq (n - 1)\log( |a|_H))$ is at least $\frac{1}{2}$. 
The probability that a particular index does not get modified in any of the $t$ iterations is $\left(1 - \frac{1}{n}\right)^{t}$. The previous statement implies that the probability that it does get modified at least once is $1 - \left(1 - \frac{1}{n}\right)^{t}$. Therefore the probability that $|a|_H$ indices gets modified at least once in $t$ iterations is $\left(1 - \left(1 - \frac{1}{n}\right)^{t}\right)^{|a|_H}$. This in turn implies that the probability that at least one of the $|a|_H$ indices does not get modified in $t$ iterations is  $1 - \left(1 - \left(1 - \frac{1}{n}\right)^{t}\right)^{|a|_H}$. If $ t = (n - 1) \ln(|a|_H)$, then 
\begin{align*}
  1 - \left(1 - \left(1 - \frac{1}{n}\right)^{t}\right)^{|a|_H} &= 1 - \left(1 - \left(1 - \frac{1}{n}\right)^{(n - 1) \ln |a|_H}\right)^{|a|_H}  \\
  &\geq 1 - \left(1 - \left(\frac{1}{e}\right)^{\ln |a|_H}\right)^{|a|_H} \\
  &= 1 - \left(1 - \frac{1}{|a|_H}\right)^{|a|_H} \geq 1 - e^{-1} \geq \frac{1}{2}.\
\end{align*}
Now we have expected time
\begin{align*}
    E[T] 
    &= \sum_{t = 1}^{\infty} t \cdot P(T = t) \\
    &= \sum_{t = 1}^{\infty} P(T \geq t) \\
    &\geq (n - 1) \ln |a|_H \cdot P(T \geq (n - 1) \ln |a|_H)\\
    &\geq (n - 1) \ln |a|_H \cdot \frac{1}{2} = \Omega(n \ln |a|_H).
\end{align*}
\end{proof}

\section{Heavy-Tailed Mutation Strength}\label{sec:heavy_tailed_operator}

In this section we discuss the behavior of the \ooea with the heavy-tailed mutation operator on optimizing any $f_a \in \mathscr{F}$. First, in the following theorem, we give an \emph{upper} bound on the expected optimization time.

\begin{thm} \label{thm:heavy_tailed_operator}
Let $f_a \in \mathscr{F}$. Then the expected optimization time of the \ooea, starting with the all-$0$s integer string, with the \emph{heavy-tailed operator} (see Definition \ref{defn:heavy_tailed_operator}) with parameter $\epsilon > 0 $ on $f_a$ is $O(n \cdot \log^2 |a|_1 \cdot \left(\log (\log |a|_1)\right)^{1 + \epsilon})$.
\end{thm}
\begin{proof}
Our proof idea is similar to the proof idea in \cite[~Theorem 15]{timo2017}. We use multiplicative drift analysis in this proof.

Let $x$ and $x'$ be the integer string at iteration $t$ and $t+1$ when the \ooea with heavy-tailed operator is optimizing $f_a$. Let the potential value $X$ at time $t$ be $f_a(x)$. Then the initial potential value is $|a|_1$, since we start with all $0$ integer string. Let $T = \min\{t \geq 0 \mid X = 0\}$. For any $t \geq 0$ and $i \in \{1, \cdots, n\}$, let $d_i = |a_i - x_i|$. 

For any $i \in \{1, \cdots, n\}$ and $j \in \{0, 1, \cdots, \floor{\log d_i}\}$, let $A_{i, j}$ be the event that the mutation operator only modifies the index $i$ such that $|a_i - x_i| - |a_i - x_i'| = 2^j$ and do not make any other changes. If $j^* = j + 2$, then $P(A_{i, j}) \geq \frac{1}{2ec_{\epsilon}nj^* (\log j^*))^{1 + \epsilon}}$. We get the previous bound on the probability because the probability to make $2^j$ changes to a particular index in the right direction is $\frac{1}{2c_{\epsilon}nj^* (\log j^*))^{1 + \epsilon}}$ and the probability to make exactly this change and no other changes to any other indices is at least $1/e$. Also note that while calculating the probability $P(A_{i, j})$, we did not consider the case that the mutation operator can overshoot, since this will only increase the probability.

\begin{align*}
E[X - X' &\mid X] \geq \sum_{i = 1}^{n} \sum_{j = 0}^{\floor{\log d_i}} E[X - X' \mid A_{i, j}, X] \cdot P(A_{i,j}) \\
&= \sum_{i = 1}^{n} \sum_{j = 0}^{\floor{\log d_i}} 2^j \cdot P(A_{i,j}) \\
&\geq  \sum_{i = 1}^{n} \sum_{j = 0}^{\floor{\log d_i}} \frac{2^j}{2ec_{\epsilon}nj^* (\log j^*))^{1 + \epsilon}}\\
&\geq \frac{1}{2ec_{\epsilon}n} \sum_{i = 1}^{n} \frac{d_i}{(\floor{\log d_i} + 2) \left(\log (\floor{\log d_i} + 2)\right)^{1 + \epsilon}}\\
% &\geq \frac{1}{2ec_{\epsilon}n} \sum_{i = 1}^{n} \frac{d_i^t}{\left(\log (\log d_i^t + 2)\right)^{1 + \epsilon}} \sum_{j = 0}^{\log d_i^t } \frac{1}{(\log d_i^t + 2) }\\
% &= \frac{1}{2ec_{\epsilon}n} \sum_{i = 1}^{n} \frac{d_i^t}{\left(\log (\log d_i^t + 2)\right)^{1 + \epsilon}} \cdot \frac{\log d_i^t + 1}{\log d_i^t + 2 }\\
% &\geq \frac{1}{2ec_{\epsilon}n} \sum_{i = 1}^{n} \frac{d_i^t}{\left(\log (\log d_i^t + 2)\right)^{1 + \epsilon}}\\
&\geq \frac{\sum_{i = 1}^{n} d_i}{2ec_{\epsilon}n \cdot (\log |a|_1 + 2) \cdot \left(\log (\log |a|_1 + 2)\right)^{1 + \epsilon}} \\
&= \frac{X}{2ec_{\epsilon}n \cdot (\log |a|_1 + 2) \cdot \left(\log (\log |a|_1 + 2)\right)^{1 + \epsilon}}.
\end{align*}

Since we have an initial potential value $|a|_1$ and a multiplicative drift value of $\frac{1}{2ec_{\epsilon}n \cdot (\log |a|_1 + 2) \cdot \left(\log (\log |a|_1 + 2)\right)^{1 + \epsilon}}$, by multiplicative drift theorem (\Cref{thm:multiplicativedrift}),
\begin{align*}
E[T] &\leq 2ec_{\epsilon}n \cdot (\log |a|_1 + 2) \cdot \left(\log (\log |a|_1 + 2)\right)^{1 + \epsilon} \cdot (1 + \log |a|_1) \\
& = O(n \cdot \log^2 |a|_1 \cdot \left(\log (\log |a|_1)\right)^{1 + \epsilon}).
\end{align*}

Thus we get the upper bound as claimed.
\end{proof}

As a corollary to the proof, we give an upper bound on the expected time taken by the \ooea with the heavy-tailed mutation operator to find an integer string which is at a distance at most $|a|_1 \cdot \alpha$ from the optimum. Thus, we can inform about the time it takes to approximate the optimum.

\begin{cor} \label{cor:heavy_tailed_operator}
Let $f_a \in \mathscr{F}$ and $\alpha \in (0, 1)$. Then the expected optimization time of the \ooea, starting with all $0$ integer string, with \emph{heavy-tailed operator} with parameter $\epsilon > 0 $ on $f_a$, to find an integer string with weight $|a|_1 \cdot \alpha$ is $O(n \cdot \log |a|_1 \cdot \log \left(\frac{1}{\alpha}\right)\cdot \left(\log (\log |a|_1)\right)^{1 + \epsilon})$.
\end{cor}

\begin{proof}
The proof is similar to the proof of the Theorem \ref{thm:heavy_tailed_operator}.
If we consider the same potential as in Theorem \ref{thm:heavy_tailed_operator}, $X = f_a(x)$, then we have the same value $\frac{1}{2ec_{\epsilon}n \cdot (\log |a|_1 + 2) \cdot \left(\log (\log |a|_1 + 2)\right)^{1 + \epsilon}}$ as the multiplicative drift. Let $T =\{t \geq 0 \mid X_t \leq |a|_1 \cdot \alpha \}$.  The initial potential value is at most $|a|_1$ and the minimum value the potential can take is $|a|_1 \cdot \alpha$. 

Therefore, by multiplicative drift theorem (\Cref{thm:multiplicativedrift}),
\begin{align*}
E[T] &\leq 4ec_{\epsilon}n \cdot (\log |a|_1) \cdot \left(\log (\log |a|_1)\right)^{1 + \epsilon} \cdot \left(1 + \log \left(\frac{|a|_1}{|a|_1 \cdot \alpha}\right) \right)\\
& = O(n \cdot \log |a|_1 \cdot \log \left(\frac{1}{\alpha}\right)\cdot \left(\log (\log |a|_1)\right)^{1 + \epsilon}).
\end{align*}

\end{proof}

\section{Self-Adjusting Mutation Rates} \label{sec:self_adjusting_rls}
In this section, we analyze self-adjusting mutation rates for the $RLS$ algorithm and show how these can outperform the \ooea with the static operators analyzed in the previous sections. The mutation strength for $RLS_{\alpha,\beta}$ is adjusted using the constants $1<\alpha\leq 2$ and $1/2<\beta<1$ (see \Cref{rls} for further details on the algorithm). In Theorem \ref{thm:self_adjusting}, we give a tight bound on the expected run time of $RLS_{\alpha,\beta}$ for suitable $\alpha$ and $\beta$. We start by giving a lower bound on the expected run time in Lemma \ref{lem:self_adjusting_lower}.

% Our main result is the matching lower bound of \Cref{lem:self_adjusting_lower} with the upper bound of \Cref{biglemma}.
\begin{lem}[RLS lower bound] \label{lem:self_adjusting_lower}
Let $f_a \in \mathscr{F}$. For constants $\alpha,\beta$ the expected optimization time of $RLS_{\alpha,\beta}$ starting with all $0$ integer string on $f_a$ is $\Theta(n \cdot \log (|a|_1))$.
\end{lem}
\begin{proof}
A bound of $\Omega(n\log |a|_H)$ easily follows from a coupon collector argument: Since we need to change $|a|_H$ many entries and each one has a probability of $\frac{1}{n}$ of being changed in each iteration.

A bound of $\Omega(n\log |a|_{\infty})$ follows from analyzing the entry $j$ with the highest distance to the target. First observe that since the velocity doubles at most each time that an entry is selected and we start at $0$, we need at least $\log(a_{\infty})-1$ changes for this entry.
Let $X_t$ be the random variable counting the number of changes on~$j$.
Let $Y_t$ be another random variable with $Y_t\coloneqq \log(a_{\infty})-1-X_t$. Since $Y_0=\log(|a|_{\infty})-1$ and we have an additive drift of at most $\frac{1}{n}$, the expected run time is of order $\Omega(n\cdot\log(|a|_{\infty}))$. We obtain this drift because the probability of changing the entry $j$ is $\frac{1}{n}$ and we can only change it or not change it.

The lower bound of $\Omega(n\log|a|_1)$ is obtained by adding both run times (this is asymptotically the same as taking the max) to get $$n\log |a|_{\infty}+n\log |a|_H=n\log(|a|_{\infty}|a|_H)\geq n\log(|a|_1).$$
\end{proof}

The proof of the upper bound in the following lemma is essentially the same as in \cite[Theorem 17]{timo2017}, only omitting parts that are not necessary for our setting. For the sake of self-containment, we present the modified proof here.

\begin{lem}[RLS upper bound] \label{biglemma}
Let $f_a \in \mathscr{F}$. For constants $\alpha,\beta$ satisfying $1<\alpha\leq 2,1/2<\beta\leq 0.9,2\alpha\beta-\beta-\alpha>0,\alpha+\beta>2$ and $\alpha^2\beta>1 $ the expected optimization time of $RLS_{\alpha,\beta}$ starting with all $0$ integer string on $f_a$ is $O(n \cdot \log (|a|_1))$.
\end{lem}

\begin{proof}

To simplify the notation for a given search point $x$ and the target integer string $z$ and the chosen metric $d$, we let $d_i=d(x_i,z_i)$ for all $(i\leq n)$ be the distance vector of $x$ to $z$. Thus, the goal is to reach a state in which the distance vector is $(0,...,0)$. We now want to define a potential function in dependence on $(d,v)$ (where of course $d$ is dependent on $x$) such that it is $0$ when $d$ is $(0,...,0)$ and strictly positive for any $x\neq (0,...,0)$.

We use as potential function the following map $g:\mathbb{Z}^n\mapsto \mathbb{R}, (x,v)\mapsto \sum_{i=1}^n g_i(d_i,v_i)$ where $g_i(d_i,v_i)\coloneqq 0$ for $d_i=0$ and for $d_i\geq 1$
\begin{align*}
    g_i(d_i,v_i) \coloneqq  d_i + \begin{cases}
			cd_i\max\{2v_i/d_i,d_i/(2v_i)\},& \text{ if $v_i\leq 2\beta d_i$};\\
            cd_i\max\{2v_i/d_i,d_i/(2v_i)\}+pd_i, & \text{otherwise.}
		 \end{cases}
\end{align*}
and $c,p$ are (small) constants specified below. For further motivation on the potential see \cite[~Theorem 17]{timo2017}.

Summarizing all the conditions needed below, we require that the constants $\alpha,\beta,c,p$ satisfy $1<\alpha\leq 2,1/2<\beta\leq 0.9, 2\alpha\beta-\beta-\alpha>0,\alpha+\beta>2,\alpha^2\beta>1,8\alpha\beta c+2p+4c/\beta\leq 1/16, p > 8c\left(\frac{\alpha+\beta}{2}-1\right),$ and $p>4(\alpha-1)c>0$.

We can thus choose, for example, $\alpha=1.7, \beta=0.9, p=0.01,$ and $c=0.001$.

Let $d\neq (0,...,0)$ and $v\in \mathbb{N}^n$. Let $(d',v')$ be the state of \Cref{rls} started in $(d,v)$ after one iteration (i.e., after a possible update of $x$ and $v$). First we show that the expected difference in potential satisfies
\begin{align*}
    E[g(d,v)-g(d',v')|d,v]\geq \frac{\delta}{n} g(d,v)
\end{align*}
for some positive constant $\delta$. Any fixed index $i$ is chosen by \Cref{rls} for mutation with probability $1/n;$ for all $i$, let $A_i$ be the event that index $i$ was chosen. We show that there is a constant $\delta$ such that, for all indices $i$ with $d_i\neq 0$
\begin{align*}
    E[g(d_i,v_i)-g(d_i',v_i') \mid d,v]\geq \delta g_i(d_i,v_i)
\end{align*}
thus proving the claim using $P(A_i)=1/n$.

We regard several cases, depending on how $d_i$ and $v_i$ relate.

Case $1$: $v_i \leq d_i/8$.\\
First we observe that $\max\{2v_i/d_i,d_i/(2v_i)\}=d_i/(2v_i)$. The contribution of the $i$-th position to the current potential is thus
\begin{align*}
    g_i(d_i,v_i) = d_i + cd_i^2/(2v_i).
\end{align*}
With probability $1/2$ the algorithm decides to move in the right direction. In this case we make progress with respect to the fitness function and the velocity. That is, after the iteration we have $d_i'=d_i - \floor{v_i}<d_i$ and $v_i' = \alpha v_i> v_i$.

To bound the progress in the second component of $g_i,$ we observe that
\begin{align*}
    cd_i'\max\{2\alpha v_i / d_i', d_i'/(2\alpha v_i)\}&=\max\{2c\alpha v_i, cd_i'^2/(2\alpha v_i)\}\\&=cd_i'^2/(2\alpha v_i),
\end{align*}
where the second equality follows from $2\alpha v_i\leq d_i/2 < d_i'$. We thus obtain that for this case the difference in potential is at least
\begin{align}\label{case1_1}
    g_i(d_i,v_i)-g_i(d_i',v_i') &= d_i + cd_i^2/(2vi)-d_i' - cd_i'^2 /(2\alpha v_i)\\ &\geq \frac{c d_i^2}{2v_i}-\frac{cd_i^2}{2\alpha v_i}.
\end{align}

With probability $1/2$ the algorithm decides to go in the wrong direction, then $d_i' > d_i$ holds and the new individual is discarded while the velocity $v_i$ at position $i$ is further decreased to $\max\{\beta v_i,1\}\geq \beta v_i$. Hence the difference in potential for this case is at least
\begin{align}\label{case1_2}
    g_i(d_i,v_i)-g_i(d_i,\beta v_i) = \frac{cd_i^2}{2v_i}-\frac{cd_i^2}{2\beta v_i}.
\end{align}

Combining (\ref{case1_1}) and (\ref{case1_2}), we thus obtain that the expected difference in potential is at least
\begin{align*}
    &\frac{1}{2}\left(\frac{cd_i^2}{2v_i}-\frac{cd_i^2}{2\alpha v_i}+\frac{cd_i^2}{2v_i}-\frac{cd_i^2}{2\beta v_i}\right)=\frac{cd_i^2}{2v_i}\left(\frac{2\alpha\beta-\beta-\alpha}{2\alpha\beta}\right)\\
    &= \left(\frac{2\alpha\beta-\beta-\alpha}{4\alpha\beta}\right)\left(\frac{cd_i^2}{2v_i}+\frac{cd_i^2}{2v_i}\right)\geq \left(\frac{2\alpha\beta-\beta-\alpha}{4\alpha\beta}\right)\left(4cd_i + \frac{cd_i^2}{2v_i}\right)\\
    &\geq \left(\frac{2\alpha\beta-\beta-\alpha}{4\alpha\beta}\right)\min\{4c,1\}\left(d_i+\frac{cd_i^2}{2v_i}\right)\\
    &=  \left(\frac{2\alpha\beta-\beta-\alpha}{4\alpha\beta}\right)\min\{4c,1\} g_i(d_i,v_i)
\end{align*}
where in the third step we have used the requirement that $v_i \leq d_i/8$.\\
Case $2$: $d_i/8 < v_i \leq 2\beta d_i$.\\
Now we are in a range of velocity which is well-suited to make progress. In fact, every step towards the optimum decreases the distance to the optimum by at least the minimum of $\floor{d_i}/{8}$ (if $v_i$ is close to $d_i/8$ and we hence do not overshoot the target) and $\floor{(2-2\beta)d_i}$ (if $v_i=2\beta d_i\geq d_i$ in which case we overshoot the target and the distance to it from $d_i$ to at most $\floor{2\beta d_i}-d_i)$. In case of moving towards the target value, the change in the first term of $g_i$ is thus at least
$$\min\{\floor{d_i/8},\floor{(2-2\beta)d_i}\}=\floor{d_i/8},$$
using $\beta \leq 0.9$. However, note that the decrease is at least $1$ (since $v_i$ is at least $1$). Furthermore, we have, for all $z\geq 8, z/16\leq \floor{z/8}.$ Thus, we always have a decrease of at least $d_i/16$.

We now compute the change in the second term of $g_i$. Regard first the case that $\max\{2v_i'/d_i',d_i'/(2v_i')\} = 2v_i'/d_i'$. In this case, we pessimistically assume that the previous contribution of the second term in $g_i(d_i,v_i)$ was zero. This contribution increases to at most\begin{align}
    2cv_i' + p d_i' \leq 2\alpha cv_i + pd_i'\leq 2\alpha cv_i + pd_i\leq (4\alpha\beta c+p)d_i.
\end{align}
On the other hand, $\max\{2v_i'/d_i',d_i'/(2v_i')\} = d_i'/(2v_i')$ and the previous contribution of the second term in $g_i(d_i,v_i)$ was $cd_i^2/(2v_i)$, then the contribution of this second term has been decreased to $c(d_i')^2/(2\alpha v_i)\leq cd_i^2/(2v_i)$. The change in contribution is thus positive in this case, and therefore in particular strictly larger than $-(4\alpha \beta c+p)d_i$. We finally need to regard the case that $$\max\{2v_i'/d_i',d_i'/(2v_i')\}=d_i'/(2v_i')$$ and $$\max\{2v_i/d_i,d_i/(2v_i)\}=2v_i/d_i.$$ In this case the contribution in the second term of $g_i$ increases by at most
\begin{align*}
    \frac{cd_i'^2}{2\alpha v_i}\leq \frac{cd_i^2}{2\alpha(d_i/8)}\leq \frac{4cd_i}{\alpha}\leq 4\alpha\beta c d_i,
\end{align*}
where the last step follows from $\alpha^2\beta\geq 1$.

Summarizing this discussion, we see that in case of stepping towards the target the change in progress satisfies
\begin{align}
    g_i(d_i,v_i)-g_i(d_i',v_i')\geq d_i(1/16 -(4\alpha\beta c+p)),
\end{align}
which is positive by our conditions on $c$ and $p$.

Let us now regard the case of stepping away from the optimum, which happens with probability $1/2$ and the velocity is decreased to $\max\{\beta v_i,1\}$. Assume first that $\max\{\beta v_i,1\} = \beta v_i$. Then,
\begin{align}\label{case2,2}
    g_i(d_i,v_i)-g_i(d_i',v_i')=\max\{2cv_i,\frac{cd_i^2}{2v_i}\}-\max\{2c\beta v_i,\frac{cd_i^2}{2\beta v_i}\}.
\end{align}
If $\max\{2c\beta v_i,cd_i^2/(2\beta v_i)\} = cd_i^2/(2\beta v_i),$ then the term in (\ref{case2,2}) is at least $-cd_i^2/(2\beta v_i)\geq -4cd_i/\beta$ by our condition $d_i/8\leq v_i$. Furthermore, if $\max\{2c\beta v_i,cd_i^2/(2\beta v_i)\}=2c\beta v_i,$ then (\ref{case2,2}) is strictly positive as can be seen by the following observation\begin{align}
    \max\{2cv_i,\frac{cd_i^2}{2v_i}\}-2c\beta v_i \geq 2cv_i - 2c\beta v_i > 0.
\end{align}
Putting everything together we thus obtain that for $d_i/8\leq v_i \leq 2\beta d_i$
\begin{align}
    E[g_i(d_i,v_i)-g_i(d_i',v_i')]\leq \frac{d_i}{2}(1/16-2(4\alpha\beta c+p)-4c/\beta)
\end{align}
which is positive if $8\alpha\beta c+2p+4c/\beta\leq 1/16$. Since $v_i= \Theta(d_i)$ this also shows that there is a positive constant $\delta$ such that $E[g_i(d_i,v_i)-g_i(d_i',v_i')]\geq \delta g_i(d_i,v_i)$.

We finally need to regard the case that $\max\{\beta v_i,1\}=1$. Intuitively, the cap can only make our situation better. This is formalized by the following computations. We need to bound
\begin{align}\label{case2,3}
    g_i(d_i,v_i)-g_i(d_i',v_i')=\max\{2cv_i,\frac{cd_i^2}{2v_i}\}-\max\{2c,\frac{cd_i^2}{2}\}.
\end{align}
As above we obtain positive drift for the case $\max\{2c,cd_i^2\}=2c$ by observing that $\max\{2cv_i,\frac{cd_i^2}{2v_i}\}-2c\geq 2cv_i-2c\geq 0$ (using that $v_i\geq 1$). For the case $\max\{2c,cd_i^2\} = cd_i^2$ the term in (\ref{case2,3}) is at least $-cd_i^2\geq -cd_i^2/(2\beta v_i)\geq -4cd_i/\beta$ as above. The same computation as above thus shows a positive multiplicative gain in $g_i$.

Case $3$: $2\beta d_i< v_i < 2d_i$.\\
Under these conditions $g_i(d_i,v_i)=d_i+2cv_i+pd_i$ holds.

As before, we first regard the case that the algorithm moves towards the target value. Since $\beta\geq 1/2$ it holds that $d_i\leq 2\beta d_i < v_i$ and the target value is thus overstepped. However, due to the requirement $v_i < 2d_i,$ the distance of the offspring is strictly smaller than the previous distance. The velocity is hence increased to $\alpha v_i$.

With probability $1/2$ the algorithm does a step away from the goal and thus the velocity is reduced to $v_i'=\max\{\beta v_i,1\}$. Regard first the case that $v_i'=\beta v_i$. Then, due to $\beta v_i < 2\beta d_i$, the penalty term $pd_i$ is no longer applied and the resulting potential at component $i$ is thus $g_i(d_i',v_i')=d_i+2c\beta v_i$.

Ignoring any possible gains in $d_i$, we therefore obtain that the expected difference in the potential is at least$$2cv_i\left(1-\frac{\alpha+\beta}{2}\right)+\frac{p}{2}d_i$$
Note that $1-(\alpha+\beta)/2$ is negative, since we require $\alpha+\beta>2.$ Using $v_i\leq 2d_i$ we see that the drift is at least
$$4cd_i\left(1-\frac{\alpha+\beta}{2}\right)+\frac{p}{2}d_i=d_i\left(\frac{p}{2}-4c\left(\frac{\alpha+\beta}{2}-1\right)\right).$$
Since $p>8c(\frac{\alpha+\beta}{2}-1)$ this expression is positive. Furthermore, we have $g_i(d_i,v_i)=\Theta(d_i),$ yielding the desired multiplicative drift.

For $v_i'= 1$ we first observe that $v_i'=1\leq d_i \leq 2\beta d_i$ and the penalty term $pd_i$ is thus not in force. Furthermore, we have $\beta d_i < \beta v_i\leq 1$ and thus $d_i\leq 1/\beta \leq 2,$ showing that $\max\{2/d_i,d_i/2\}\leq \max\{2,1\}=2.$ We obtain
$$E[g_i(d_i,v_i)-g(d_i',v_i')\geq\frac{p}{2}d_i-cv_i(\alpha-1)\geq \frac{p}{2}d_i-2(\alpha-1)cd_i,$$
which is positive for $p/2-2(\alpha-1)c>0$.

Case $4$: $v_i=2d_i.$ \\
Steps away from the target are not accepted, thus regardless of whether or not we move towards or away from the target, the fitness does not decrease; therefore, the velocity is decreased to $\beta v_i$ (note that $v_i\geq 2$ and hence $\beta v_i\geq 1)$. The previous contribution of the $i$-th component to $g(x)$ being $d_i+2cv_i+pd_i=d_i(1+4c+p),$ and the new potential at the $i$-th component being $d_i(1+4\beta c)$, we obtain
$$E[g_i(d_i,v_i)-g_i(d_i',v_i')]=d_i(4c+p-4\beta c),$$
which is strictly positive and linear in $g_i(d_i,v_i)$.

Case $5$: $2d_i<v_i$.\\
Steps towards the optimum are now also not accepted, since they overstep the optimum by too much. Therefore, we always decrease the velocity to $\max\{\beta v_i,1\}=\beta v_i$ (note that $v_i>2$ and thus $\beta v_i >1$) and the gain in potential is
$$2cv_i + pd_i - (2c\beta v_i+pd_i)=2cv_i(1-\beta)>4(1-\beta)cd_i$$
showing that we have the multiplicative drift as desired.

Together with the observation that
the initial potential is of order at most 
$$\sum_{i=1}^n d_i^2 \leq \left(\sum_{i=1}^n d_i\right)^2 = |a|_1^2 $$ plugged into the multiplicative drift theorem (\Cref{thm:multiplicativedrift}) proves the desired overall expected run time of $O(n \log(|a|_1))$.
\end{proof}  
Combining both results we get a sharp run time result in the following theorem.

\begin{thm} \label{thm:self_adjusting}
Let $f_a \in \mathscr{F}$. For constants $\alpha,\beta$ satisfying $1<\alpha\leq 2,1/2<\beta\leq 0.9,2\alpha\beta-\beta-\alpha>0,\alpha+\beta>2$ and $\alpha^2\beta>1 $ the expected optimization time of $RLS_{\alpha,\beta}$ starting with all $0$ integer string on $f_a$ is $\Theta(n \cdot \log (|a|_1))$.
\end{thm}
\begin{proof}
    The result follows by the matching lower bound in \Cref{lem:self_adjusting_lower} with the upper bound in \Cref{biglemma},
\end{proof}
\section{Comparison with {CMA-ESwM}}
\label{sec:experiments}

In this section we provide the following empirical analyses. First, we show that the CMA-ESwM is not efficient for the integer-valued optimization problem we consider. Then we compare the performances of \ooea with the $\pm 1$ operator, \ooea with the heavy-tailed operator and the RLS with the self-adjusting operator. 

Note that although the CMA-ESwM is designed to optimize mixed integer valued problems, we restrict ourselves to all-integer inputs. This allows us to make comparisons with the \ooea and RLS with different mutation operators. We use the code from GitHub\footnote{\url{https://github.com/EvoConJP/CMA-ES_with_Margin}} provided by the authors of the paper where the CMA-ESwM is proposed \cite{CMAESwMargin}.

% \Acomment{Sherry, "therefore we let the continuous part of the algorithm be $0$ so that we can compare run time of different algorithms used for discrete optimization problems", we removed this statement because we are not getting into the details of the implementation of the algorithm and "let the continuous part of the algorithm be $0$" is not very clear. Is this okay?} 

All our theoretical analyses are concerned with the unbounded integer search space, whereas, for practical reasons, we have to restrict our search space to be bounded. We bound the maximum value of the considered integer strings by a value $r$.

% We are interested in the run time of finding the optimum of four algorithms, namely \ooea with $\pm 1$ operator,\ooea with the heavy-tailed operator,  \rlswself with the self-adjusting operator and CMA-ESwM in the sense of different sizes of search space $r$ experimentally. \Scomment{Aishwarya point out for discrete algorithms there aren't any bounds so she points out need to mention we need a bound for discrete algorithm. I think 'in the sense of different size of search space $r$ experimentally' give an idea of r.}

We set the optimum as the all-$r$ integer string and let the algorithm run until it finds the optimum and record the run time (number of function evaluations). For \ooea with different mutation operators and \rlswself, we start with the all-$0$ integer string. As for CMA-ESwM, we consider the same set up proposed in \cite{CMAESwMargin}. 

We choose the step size of $r$ as follows: 
\begin{enumerate}
    \item \text{\ooea} with $\pm 1$ operator: $r$ from 10 till 150 with a step size of 10. Then we consider $r \in \{ 10^{3}, 10^{4}, 10^{5} \}$ to analyze the run time of the algorithm for exponentially increasing values of~$r$. 
    \item  \rlswself with the self-adjusting operator and  \ooea with the heavy-tailed operator: we consider $r = 10^{k}$, for $k \in \{1, \ldots, 12\}$.
    % $r$ from $10$ till $10^{12}$ with an exponential increase step size, which means 
    For RLS with the self-adjusting operator we set the parameter $\alpha = 2.0$ and $\beta = 0.5$ and for the \ooea with the heavy-tailed operator we set $\epsilon = 0.001$.
    \item CMA-ESwM: we consider the same values of $r$ as for the RLS since there are no specific restrictions mentioned in~\cite{CMAESwMargin}. 
\end{enumerate}

\subsection{Success rate}
During initial exploration we noticed that the CMA-ESwM does not always find the optimum before it reaches the termination condition (minimum eigenvalue of the covariance matrix in the update step of the CMA-ESwM algorithm is less than $10^{-30}$). To further analyze this, we consider the following different values of $r$, $\{10, 100, 1000\}$ and different $n$ (length of the integer string) from $10$ to $100$ in steps of $10$. In Figure~\ref{fig:failedRateCMAESwM} we see the proportion of failed runs over total of $100$ runs named failure rate. Note that for $n < 40$ the failure rate is zero so in Figure \ref{fig:failedRateCMAESwM}, $n \geq 40$. 

\begin{figure}
          \includegraphics[width= 0.45\textwidth]{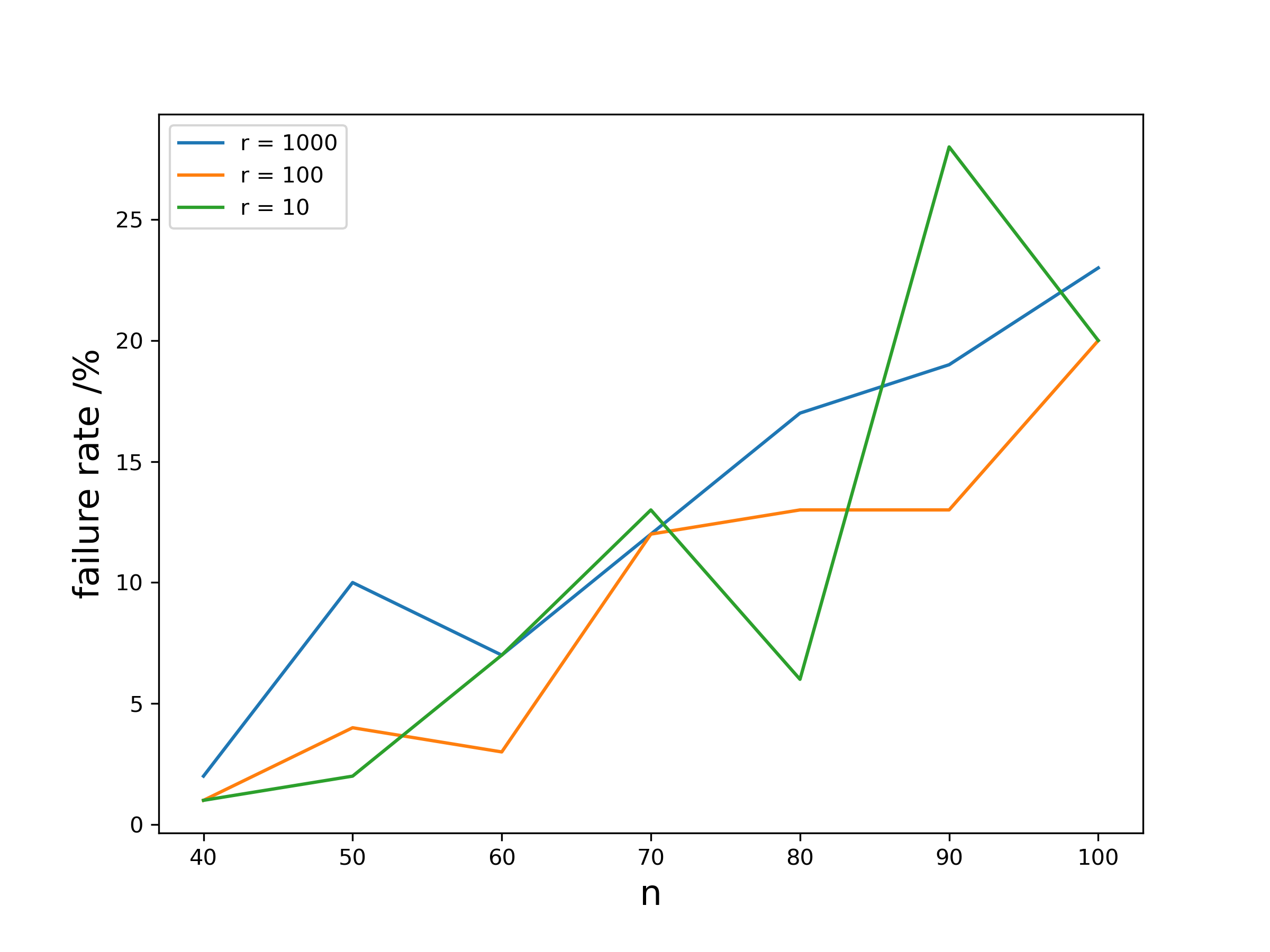}
    \caption{Failure rate of the CMA-ESwM.}
    \label{fig:failedRateCMAESwM}
\end{figure}

We can see from Figure ~\ref{fig:failedRateCMAESwM} that the
failure rate increases as $n$ increases. Also, the failure rate is significant despite $r$ being quite small. This shows the limitations of CMA-ESwM in optimizing a complete integer-valued problem.

\subsection{Comparison between different algorithms}
We present experimental results of the \ooea with the $\pm 1$ operator and the heavy-tailed operator in Figure~\ref{fig:allCom}. 
All results are averaged over $20$ independent runs. We attach the results of the statistical test for $n = 100$ in the appendix. % and presented in plots generated using the Matplotlib library \cite{matplotlib}.

% \begin{figure}
%           \includegraphics[width= 0.45\textwidth]{simpleOperator_com20n100n.pdf}
%     \caption{Run time of \ooea with $\pm 1$ operator normalized by $n$ in linear increasing $r$.  }
%     \label{fig:simpleOperatorRuntimeOverR}
% \end{figure}

% \begin{figure}
%           \includegraphics[width= 0.45\textwidth]{heavyOpSelfOp_diffR_com.pdf}
%     \caption{Run time comparison of \rlswself with \emph{self-adjusting operator} and \ooea with \emph{heavy-tailed operator} normalized by $n$ individually in exponentially increasing $r$. Different color stands for different mutation operator and different line style stands for different $n$ value.}
%     \label{fig:heavyTailSelfAdjustCom}
% \end{figure}

% In Figure~\ref{fig:heavyTailSelfAdjustCom} we notice for $r$ up to $10^4$, the run time of the \rlswself with \emph{heavy-tailed operator} and \ooea with \emph{heavy-tailed operator} does not show an obvious difference. And until $r= 10^9$, \ooea with \emph{heavy-tailed operator} does not show the strength in short run time \Acomment{I think it is not the runtime but small value of r. And also it is the opposite, heavy-tailed operator is supposed to take the more time than RLS and it cannot be seen for small values of r}. However, after $r$ increase significantly large, the \ooea with \emph{heavy-tailed operator} increase \Acomment{What increases? I think it is important to mention that the run time increases}noticeably while \rlswself with \emph{heavy-tailed operator} \Acomment{We don't use RLS with heavy-tailed operator} maintains a steady increases as $r$. 

\begin{figure}[ht]
          \includegraphics[width= 0.45\textwidth]{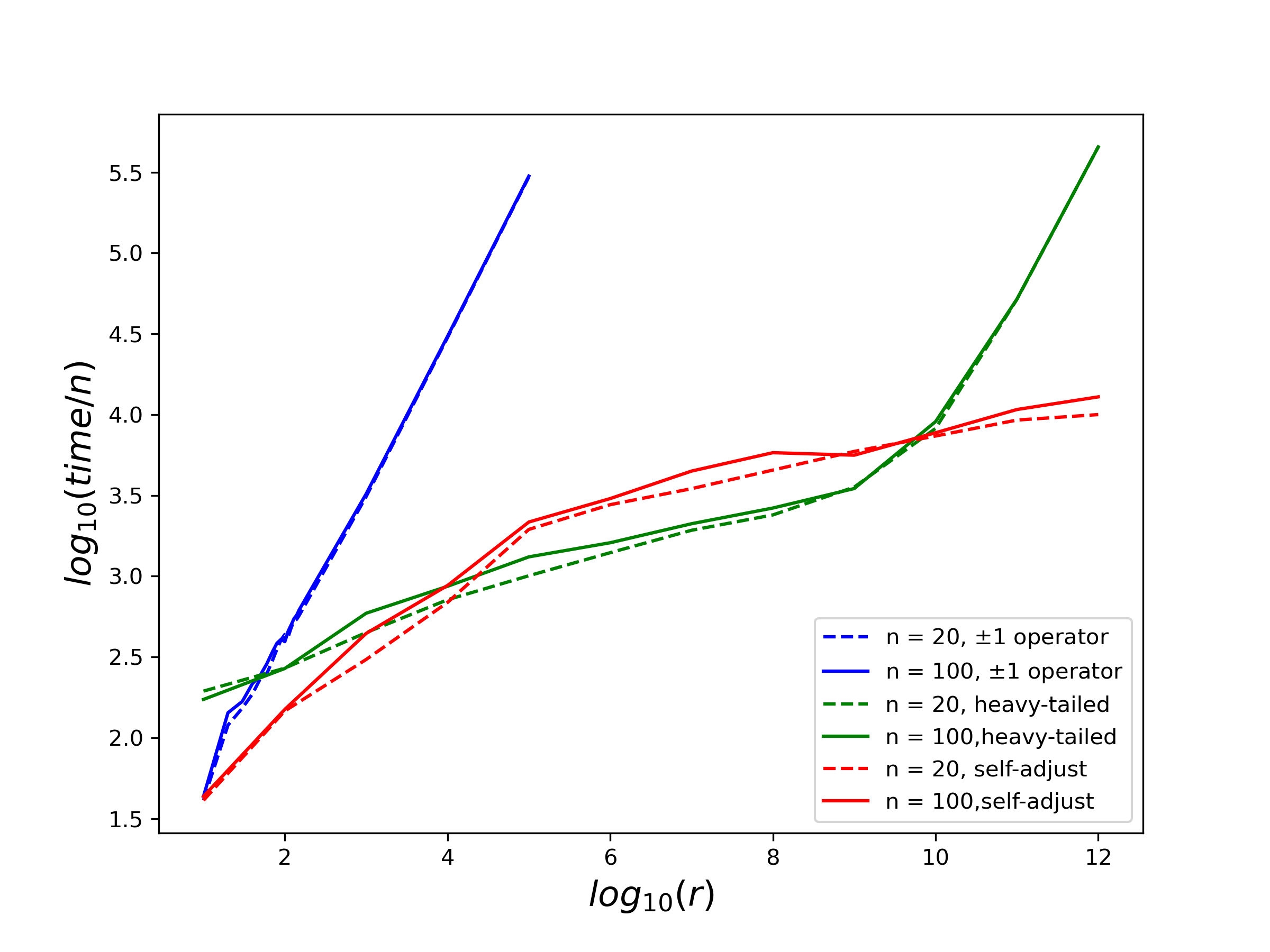}
    \caption{Run time comparison of the \ooea with two different mutation operators and the RLS with the self-adjusting operator. Different colors stand for different mutation operators and different line styles stand for different values of $n$. Note that both axes are logarithmically scaled.}
    \label{fig:allCom}
\end{figure}

% For both x-axis and y-axis, we apply $log_{10}$ for $r$ and $time$ normalized by individual $n$ for a better view

In Figure~\ref{fig:allCom} we can see that the scaling behavior with respect to $r$ is independent of the value of $n$.

Asymptotically, the results are as suggested by the theoretical results given in the prior sections. However, for small values of $r$, the scaling behavior is not yet the deciding factor. In particular, the $\pm 1$ operator is competitive as long as the optimum is not much more than $r = 10^2 = 100$ away in each component. For higher values of $r$, the constantly small step size is very much detrimental to efficient search.

An interesting finding is that the heavy-tailed operator can outperform the self-adjusting RLS, in spite of what the asymptotic bounds given in this paper suggest. For small values of $r$, a lot of time is wasted on attempting larger jumps, but for middle-ranged $r$ these jumps start to pay off. In contrast, the self-adaptive RLS needs a warm-up phase adjusting its velocity value, meanwhile the heavy-tailed operator can make progress starting in the first iteration.

% % As for comparing the run time for \ooea with heavy-tailed operator to \rlswself to self-adjusting operator, when $r$ is in the middle range, experimental result are different from \ref{thm:heavy_tailed_operator} and \ref{thm:self_adjusting}, which asymptotically \rlswself with self-adjusting operator takes shorter time to find the optimum. 

% From Figure~\ref{fig:simpleOperatorRuntimeOverR} we see a linear increase of run time as the linear increase of $r$ value, which as \ref{thm:better_lowerbound_pm_operator} and \ref{thm:better_upperbound_pm_operator} shows.

\bibliographystyle{ACM-Reference-Format}
\bibliography{main.bib}

% \section{Purgatory}

% \input{purgatory}

\begin{acks}
This work is supported by grant FR 2988/17-1 by the German Research Foundation (DFG).
\end{acks}

\appendix

\section{Statistical test for run time comparison}

In this section we present the results of the statistical test. For each algorithm, we use one box plot to show the distribution of independent runs for $n = 100$. Each box is the first quartile ($Q1$) to the third quartile ($Q3$) of the group. The whiskers extend the box by $1.5$ times the interquartile range ($IQR$). Dots are outliers which pass whiskers.

For the other set up $n = 20$, we get a similar box plot. 

\begin{figure}[ht]
        \includegraphics[width = 0.45\textwidth]{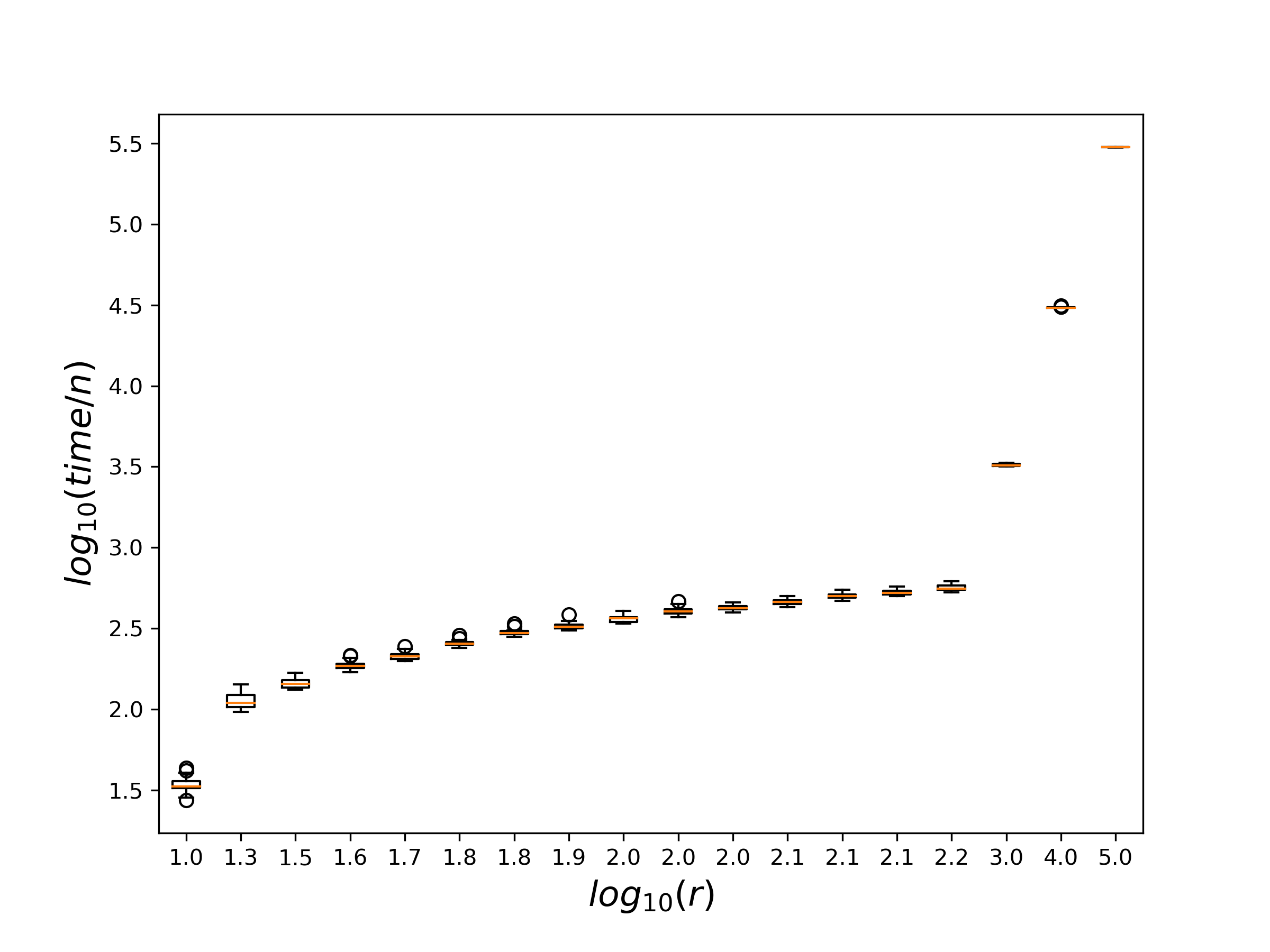}
    \caption{\ooea with $\pm 1$ operator}
    \label{fig:+-1Operator}
\end{figure}

\begin{figure}[H]
        \includegraphics[width = 0.45\textwidth]{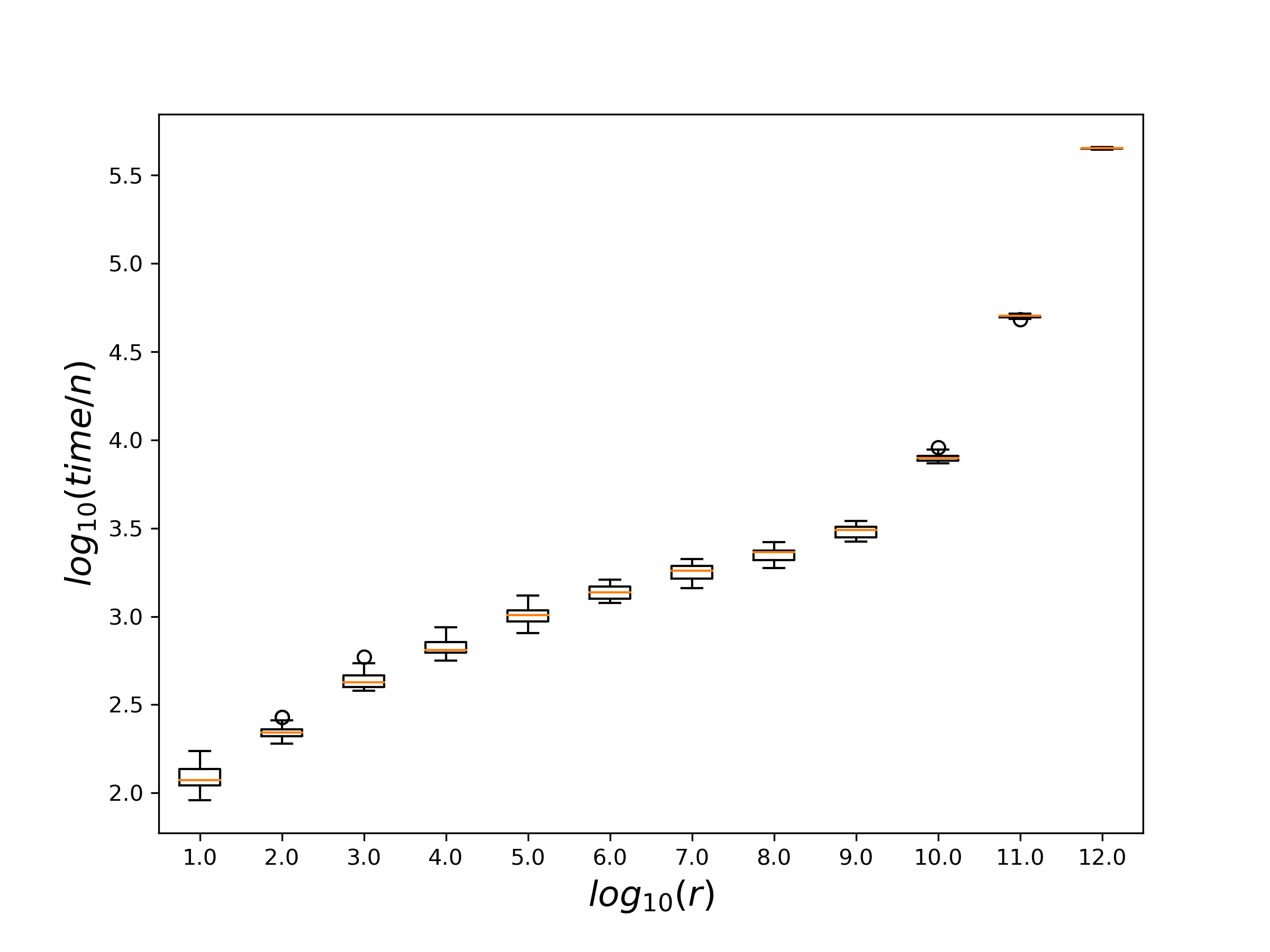}
        % \caption{$n = 800$}
    \caption{\ooea with \emph{heavy-tailed operator}.}
    \label{fig:heavy-tailedOperator}
\end{figure}

\begin{figure}[H]
        \includegraphics[width = 0.45\textwidth]{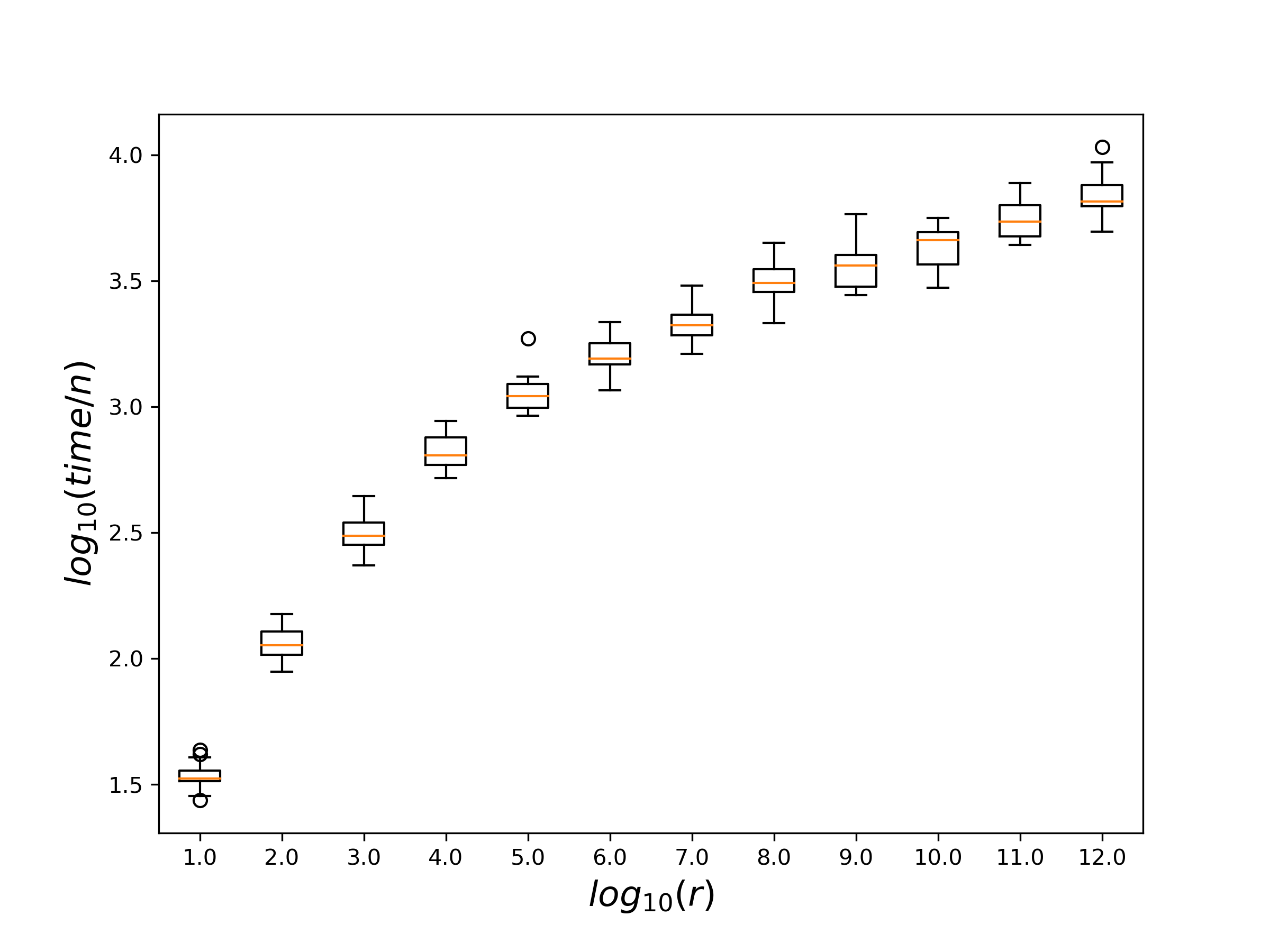}
        % \caption{$n = 800$}
    \caption{RLS with the self-adjusting operator.}
    \label{fig:self-adjustingOperator}

\end{figure}

%%%%%%%%%%%%%%%%%%%%% if use subfigure
% \usepackage{caption}
% \usepackage{subcaption}
% \begin{figure}
    
%     \begin{subfigure}{0.25\textwidth}
%         \includegraphics[width = \textwidth]{boxplot_simple_100n.pdf}
%         \caption{$n = 800$}
%     \end{subfigure}

%     \begin{subfigure}{0.25\textwidth}
        
%         \includegraphics[width=\textwidth]{boxplot_jump_100n.pdf}
%         \caption{$n = 800$}
%     \end{subfigure}

%     \begin{subfigure}{0.25\textwidth}
        
%         \includegraphics[width=\textwidth]{boxplot_self_100n.pdf}
%         \caption{$n = 800$}
%     \end{subfigure}
% \end{figure}
\end{document}